\documentclass[runningheads]{llncs}
\usepackage{graphicx}
\usepackage{amsmath,amssymb} 
\usepackage{color}
\usepackage[width=122mm,left=12mm,paperwidth=146mm,height=193mm,top=12mm,paperheight=217mm]{geometry}

\usepackage[ruled]{algorithm2e}
\usepackage{bm}
\usepackage{commath}
\usepackage{multirow}
\usepackage{soul}
\usepackage{csquotes}
\usepackage{mathtools}

\usepackage{amsthm}

\newcommand\eqdef{\stackrel{\mathclap{\normalfont\mbox{\normalfont\tiny def}}}{=}}
\newcommand{\ignore}[1]{}

\SetKwProg{Fn}{Function}{}{end}%
\usepackage{hyperref} 

\graphicspath{{figures/}}

\begin{document}
\pagestyle{headings}
\mainmatter

\title{Variational Wasserstein Clustering} 

\titlerunning{Variational Wasserstein Clustering}

\authorrunning{Liang Mi~\textit{et al.}}

\author{Liang Mi$^1$, Wen Zhang$^1$, Xianfeng Gu$^2$, and Yalin Wang$^1$}


\institute{$^1$Arizona State University, Tempe, USA\quad $^2$Stony Brook University, Stony Brook, USA\\ 
\email{ \{liangmi,wzhan139,ylwang\}@asu.edu}\quad\quad \email{gu@cs.stonybrook.edu}\\
}

\maketitle

\begin{abstract}

We propose a new clustering method based on optimal transportation. We discuss the connection between optimal transportation and k-means clustering, solve optimal transportation with the variational principle, and investigate the use of power diagrams as transportation plans for aggregating arbitrary domains into a fixed number of clusters. We drive cluster centroids through the target domain while maintaining the minimum clustering energy by adjusting the power diagram. Thus, we simultaneously pursue clustering and the Wasserstein distance between the centroids and the target domain, resulting in a measure-preserving mapping. We demonstrate the use of our method in domain adaptation, remeshing, and learning representations on synthetic and real data. 

\keywords{clustering, discrete distribution, k-means, measure preserving, optimal transportation, Wasserstein distance}
\end{abstract}

\section{Introduction}
\label{sec:intro}

Aggregating distributional data into clusters has ubiquitous applications in computer vision and machine learning. A continuous example is unsupervised image categorization and retrieval where similar images reside close to each other in the image space or the descriptor space and they are clustered together and form a specific category. A discrete example is document or speech analysis where words and sentences that have similar meanings are often grouped together. k-means~\cite{lloyd1982least,forgy1965cluster} is one of the most famous clustering algorithms, which aims to partition empirical observations into $k$ clusters in which each observation has the closest distance to the \textit{mean} of its own cluster. It was originally developed for solving quantization problems in signal processing and in early 2000s researchers have discovered its connection to another classic problem optimal transportation which seeks a transportation plan that minimizes the transportation cost between probability measures~\cite{graf2007foundations}.

The optimal transportation (OT) problem has received great attention since its very birth. Numerous applications such as color transfer and shape retrieval have benefited from solving OT between probability distributions. Furthermore, by regarding the minimum transportation cost -- \textit{the Wasserstein distance} -- as a metric, researchers have been able to compute the barycenter~\cite{agueh2011barycenters} of multiple distributions, e.g.~\cite{cuturi2014fast,solomon2015convolutional}, for various applications. Most researchers regard OT as finding the optimal coupling of the two probabilities and thus each sample can be mapped to multiple places. It is often called Kantorovich's OT. Along with this direction, several works have shown their high performances in clustering distributional data via optimal transportation, .e.g.~\cite{ye2017fast,solomon2015convolutional,ho2017multilevel}. On the other hand, some researchers regard OT as a measure-preserving mapping between distributions and thus a sample cannot be split. It is called Monge-Brenier's OT.

In this paper, we propose a clustering method from Monge-Brenier's approach. Our method is based on Gu \MakeLowercase{\textit{et al.}}~\cite{gu2013variational} who provided a variational solution to Monge-Brenier OT problem. We call it \textit{variational optimal transportation} and name our method \textit{variational Wasserstein clustering}. We leverage the connection between the \textit{Wasserstein distance} and the clustering error function, and simultaneously pursue the Wasserstein distance and the k-means clustering by using a power Voronoi diagram. Given the empirical observations of a target probability distribution, we start from a sparse discrete measure as the initial condition of the centroids and alternatively update the partition and update the centroids while maintaining an optimal transportation plan. From a computational point of view, our method is solving a special case of the \textit{Wasserstein barycenter} problem~\cite{agueh2011barycenters,cuturi2014fast} when the target is a univariate measure. Such a problem is also called the \textit{Wasserstein means} problem~\cite{ho2017multilevel}. We demonstrate the applications of our method to three different tasks -- domain adaptation, remeshing, and representation learning. In domain adaptation on synthetic data, we achieve competitive results with D2~\cite{ye2017fast} and JDOT~\cite{courty2017joint}, two methods from Kantorovich's OT. The advantages of our approach over those based on Kantorovich's formulation are that (1) it is a local diffeomorphism; (2) it does not require pre-calculated pairwise distances; and (3) it avoids searching in the product space and thus dramatically reduces the number of parameters.

The rest of the paper is organized as follows. In Section~\ref{sec:relate} and~\ref{sec:pre}, we provide the related work and preliminaries on optimal transportation and k-means clustering. In Section~\ref{sec:vot}, we present the variational principle for solving optimal transportation. In Section~\ref{sec:vwc}, we introduce our formulation of the k-means clustering problem under variational Wasserstein distances. In Section~\ref{sec:apply}, we show the experiments and results from our method on different tasks. Finally, we conclude our work in Section~\ref{sec:discuss} with future directions.

\section{Related Work}
\label{sec:relate}

\subsection{Optimal Transportation}
\label{sec:relate_ot}
The optimal transportation (OT) problem was originally raised by Monge~\cite{monge1781memoire} in the 18th century, which sought a transportation plan for matching distributional data with the minimum cost. In 1941, Kantorovich~\cite{kantorovich1942translocation} introduced a relaxed version and proved its existence and uniqueness. Kantorovich also provided an optimization approach based on linear programming, which has become the dominant direction. Traditional ways of solving the Kantorovich's OT problem rely on pre-defined pairwise transportation costs between measure points, e.g.~\cite{cuturi2013sinkhorn}, while recently researchers have developed fast approximations that incorporate computing the costs within their frameworks, e.g.~\cite{solomon2015convolutional}. 

Meanwhile, another line of research followed Monge's OT and had a breakthrough in 1987 when Brenier~\cite{brenier1991polar} discovered the intrinsic connection between optimal transportation and convex geometry. Following Brenier's theory, M\'erigot \cite{merigot2011multiscale}, Gu \textit{et al.}~\cite{gu2013variational}, and L\'evy~\cite{levy2015numerical} developed their solutions to Monge's OT problem. M\'erigot and L\'evy's OT formulations are non-convex and they leverage damped Newton and quasi-Newton respectively to solve them. Gu \textit{et al.} proposed a convex formulation of OT particularly for convex domains where pure Newton's method works and then provided a variational method to solve it.

\subsection{Wasserstein Metrics}
\label{sec:relate_wm}
The Wasserstein distance is the minimum cost induced by the optimal transportation plan. It satisfies all metric axioms and thus is often borrowed for measuring the similarity between probability distributions. The transportation cost generally comes from the product of the geodesic distance between two sample points and their measures. We refer to $p$--Wasserstein distances to specify the exponent $p$ when calculating the geodesic~\cite{givens1984class}. The $1$--Wasserstein distance or earth mover's distance (EMD) has received great attention in image and shape comparison~\cite{rubner2000earth,ling2007efficient}. Along with the rising of deep learning in numerous areas, $1$--Wasserstein distances have been adopted in many ways for designing loss functions for its superiority over other measures~\cite{lee2018wasserstein,arjovsky2017wasserstein,frogner2015learning,gibbs2002choosing}. The $2$--Wasserstein distance, although requiring more computation, are also popular in image and geometry processing thanks to its geometric properties such as barycenters~\cite{agueh2011barycenters,solomon2015convolutional}. In this paper, we focus on $2$--Wasserstein distances.

\subsection{K-means Clustering}
\label{sec:relate_km}
The K-means clustering method goes back to Lloyd~\cite{lloyd1982least} and Forgy~\cite{forgy1965cluster}. Its connections to the $1,2$-Wasserstein metrics were leveraged in~\cite{ho2017multilevel} and~\cite{applegate2011unsupervised}, respectively. The essential idea is to use a sparse discrete point set to cluster denser or continuous distributional data with respect to the Wasserstein distance between the original data and the sparse representation, which is equivalent to finding a Wasserstein barycenter of a single distribution~\cite{cuturi2014fast}. A few other works have also contributed to this problem by proposing fast optimization methods, e.g.~\cite{ye2017fast}.

In this paper, we approach the k-means problem from the perspective of optimal transportation in the variational principle. Because we leverage power Voronoi diagrams to compute optimal transportation, we simultaneously pursue the Wasserstein distance and k-means clustering. We compare our method with others through empirical experiments and demonstrate its applications in different fields of computer vision and machine learning research.


\section{Preliminaries}
\label{sec:pre}
We first introduce the optimal transportation (OT) problem and then show its connection to k-means clustering. We use $X$ and $Y$ to represent two Borel probability measures and $M$ their compact embedding space.

\subsection{Optimal Transportation}
\label{sec:pre_ot}
Suppose $\mathcal{P}(M)$ is the space of all Borel probability measures on $M$. Without losing generality, suppose $X(x,\mu)$ and $Y(y,\nu)$ are two such measures, i.e. $X$, $Y\in \mathcal{P}(M)$. Then, we have $1 = \int_{M}^{} \mu(x) dx = \int_{M}^{} \nu(y) dy$, with the supports $\Omega_{X} = \{x\} =\{m \in M \mid \mu(m) > 0\}$ and $\Omega_{Y} = \{y\} =\{m \in M \mid \nu(m) > 0\}$. We call a mapping $T:X(x,\mu) \rightarrow Y(y,\nu)$ a measure-preserving one if the measure of any subset $B$ of $Y$ is equal to the measure of the origin of $B$ in $X$, which means $\mu(T^{-1}(B)) = \nu(B),\  \forall B \subset Y$. 

We can regard $T$ as the \textit{coupling} $\pi(x,y)$ of the two measures, each being a corresponding \textit{marginal} $\mu=\pi(\cdot,y)$, $\nu=\pi(x,\cdot)$. Then, all the couplings are the probability measures in the product space, $\pi \in \prod(M \times M)$. Given a transportation cost $c:M \times M \rightarrow \mathbb{R}^{+}$ --- usually the geodesic distance to the power of $p$, $c(x,y)={d(x,y)}^p$ --- the problem of optimal transportation is to find the mapping $\pi_{opt}:x \rightarrow y$ that minimizes the total cost,
\begin{equation}
  W_p(\mu,\nu) \eqdef \bigg(\underset{\pi \in \prod(\mu,\nu)}{\inf}  \int_{M\times M}^{} c(x,y)d\pi(x,y)\bigg)^{1/p},
  \label{eq:cost_continuous}
\end{equation}
\noindent
where $p$ indicates the power. We call the minimum total cost the $p$ --\textit{Wasserstein distance}. Since we address Monge's OT in which mass cannot be split, we have the restriction that $d\pi(x,y) = d\pi_T(x,y) \equiv d\mu(x)\delta[y=T(x)]$, inferring that

\begin{equation}
  {\pi_T}_{opt} = T_{opt} = \underset{T}{\arg\min}  \int_{M}^{} c(x,T(x))d\mu(x).
  \label{eq:ot_monge}
\end{equation}

\noindent 
In this paper, we follow Eq. (\ref{eq:ot_monge}). The details of the optimal transportation problem and the properties of the Wasserstein distance can be found in~\cite{villani2003topics,gibbs2002choosing}. For simplicity, we use $\pi$ to denote the optimal transportation map.

\subsection{K-means Clustering}
\label{sec:pre_kmeans}
Given the empirical observations $\{(x_i,\mu_i)\}$ of a probability distribution $X(x,\mu)$, the k-means clustering problem seeks to assign a cluster centroid (or prototype) $y_j = y(x_i)$ with label $j=1,...,k$ to each empirical sample $x_i$ in such a way that the error function (\ref{eq:kmeans}) reaches its minimum and meanwhile the measure of each cluster is preserved, i.e. $\nu_j = \sum_{y_j=y(x_i)} \mu_i$. It is equivalent to finding a partition $V=\{(V_j,y_j)\}$ of the embedding space $M$. If $M$ is convex, then so is $V_j$.

\begin{equation}
\label{eq:kmeans}
	\underset{y}{\arg\min}\sum_{x_i}\mu_i {d(x_i,y(x_i))}^p\ \equiv\ \underset{V}{\arg\min}\sum_{j=1}^{K}\sum_{x_i \in V_j}\mu_i {d(x_i,y(V_j))}^p.
\end{equation}

Such a clustering problem (\ref{eq:kmeans}), when $\nu$ is fixed, is equivalent to Monge's OT problem (\ref{eq:ot_monge}) when the support of $y$ is sparse and not fixed because $\pi$ and $V$ induce each other, i.e. $\pi \Leftrightarrow V$. Therefore, the solution to Eq. (\ref{eq:kmeans}) comes from the optimization in the search space $\mathcal{P}(\pi,y)$. Note that when $\nu$ is not fixed such a problem becomes \textit{the Wasserstein barycenter} problem as finding a minimum in $\mathcal{P}(\pi,y,\nu)$, studied in~\cite{agueh2011barycenters,cuturi2014fast,ye2017fast}.


\section{Variational Optimal Transportation}
\label{sec:vot}
We present the variational principle for solving the optimal transportation problem. Given a metric space $M$, a Borel probability measure $X(x,\mu)$, and its compact support $\Omega = \textit{supp}\ \mu = \{x \in M \mid \mu(x) > 0\}$, we consider a sparsely supported point set with Dirac measure $Y(y,\nu)=\{(y_j,\nu_j>0)\}$, $j=1,...,k$. (Strictly speaking, the empirical measure ${X(x,\mu)}$ is also a set of Dirac measures but in this paper we refer to $X$ as the empirical measure and $Y$ as the Dirac measure for clarity.) Our goal is to find an optimal transportation plan or map (OT-map), $\pi:x\rightarrow y$, with the \textit{push-forward} measure $\pi_\#\mu=\nu$. This is \textit{semi-discrete OT}.
\begin{figure}[!t]
  \centering
    \includegraphics[width=\textwidth]{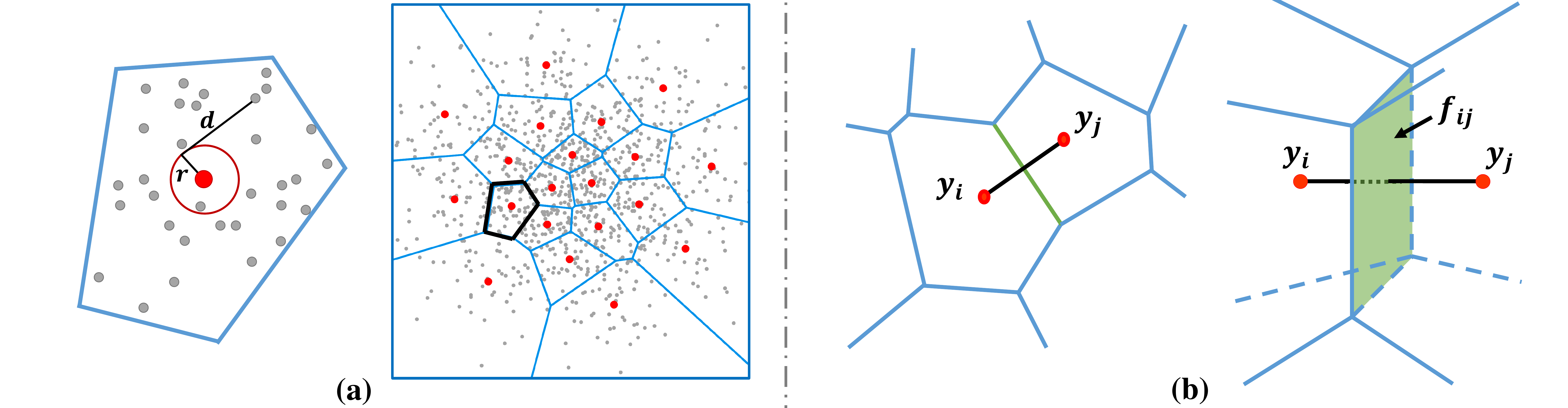}
    \caption{{\small (a) Power Voronoi diagram. Red dots are centroids of the Voronoi cells, or clusters. The Power distances has an offset depending on the weight of the cell. (b) Intersection of adjacent cells in 2D and 3D for computing Hessian. }}
\label{fig:hessian}
\end{figure}

We introduce a vector $\bm{h} = (h_1, ..., h_k)^{T}$, a hyperplane on $M$, $\gamma_j(\bm{h}):\langle m,y_j \rangle + h_j = 0$, and a piecewise linear function:
\begin{equation}
	\theta_{\bm{h}}(x) = \max\{\langle x,y_j\rangle+h_j\},\ j = 1,...,k.\nonumber
	\label{eq:linearConvex}
\end{equation}

\begin{theorem}
\label{theory:alexandrov}
(Alexandrov~\cite{alexandrov2005convex}) Suppose $\Omega$ is a compact convex polytope with non-empty interior in $\mathbb{R}^n$ and $y_1,...,y_k \subset \mathbb{R}^n$ are $k$ distinct points and $\nu_1,...,\nu_k > 0$ so that $\sum_{j=1}^{k}\nu_j = \text{vol}(\Omega)$. There exists a unique vector $\bm{h} = (h_1,...,h_k)^T \in \mathbb{R}^k$ up to a translation factor $(c,...,c)^T$ such that the piecewise linear convex function $\theta_{\bm{h}}(x)= \ignore{\underset{x\in\Omega}}{\max}\{\langle x,y_j\rangle+h_j\}$ satisfies $\text{vol}(x \in \Omega \mid \nabla \theta_{\bm{h}}(x)=y_j) = \nu_j$. 
\end{theorem}

Furthermore, Brenier~\cite{brenier1991polar} proved that the gradient map $\nabla\theta$ provides the solution to Monge's OT problem, that is, $\nabla \theta_{\bm{h}}$ minimizes the transportation cost $\int_{\Omega}\|x-\theta_{\bm{h}}(x)\|^2$. Therefore, given $X$ and $Y$, $\bm{h}$ by itself induces OT. 

From~\cite{aurenhammer1987power}, we know that a convex subdivision associated to a piecewise-linear convex function $u_h(x)$ on $\mathbb{R}^n$ equals a \textit{power Voronoi diagram}, or \textit{power diagram}. A typical power diagram on $M \subset \mathbb{R}^n$ can be represented as:
\begin{equation}
  V_j \eqdef \{ m\in M \mid \|m - y_j\|^2 - r_j^2 \leqslant \|m - y_i\|^2 - r_i^2  \},\  \forall j \neq i. \nonumber
\end{equation}
\noindent Then, a simple calculation gives us 
\begin{equation}
m \cdot y_j - \frac{1}{2}(y_j \cdot y_j + r_j^2) \leqslant m \cdot y_i - \frac{1}{2}(y_i \cdot y_i + r_i^2), \nonumber
\end{equation}
\noindent where $m \cdot y_j = \langle m,y_j \rangle$ and $w_j$ represents the offset of the \textit{power distance} as shown in Fig.~\ref{fig:hessian} (a). On the other hand, the graph of the hyperplane $\pi_j(\bm{h})$ is 
\begin{equation}
  U_i \eqdef \{ m\in M \mid \langle m,y_j \rangle -h_j \geqslant \langle m,y_i \rangle -h_i \},\  \forall j \neq i. \nonumber
  \label{eq:hyperplaneGraph}
\end{equation}
\noindent Thus, we obtain the numerical representation $h_j = -\frac{|y_j|^2-r_j^2}{2}$.

We substitute $M(m)$ with the measure $X(x)$. In our formulation, Brenier's gradient map $\nabla \theta_{\bm{h}}:V_j(\bm{h})\rightarrow y_j$ ``transports'' each $V_j(\bm{h})$ to a specific point $y_j$. The total mass of $V_j(\bm{h})$ is denoted as: $w_j(\bm{h}) = \sum_{x \in V_j(\bm{h})} \mu(x)$.

\begin{algorithm}[!t]
\SetKwFunction{FRecurs}{Variational-OT}
\Fn{\FRecurs{$X(x,\mu)$, $Y(y,v)$, $\epsilon$}}{
\DontPrintSemicolon
	$\bm{h} \leftarrow \bm{0}$.\;
	\Repeat{$|\nabla E(\bm{h})| < \epsilon$.}{
		Update power diagram $V$ with $(y,\bm{h})$.\;
		Compute cell weight $w(\bm{h}) = \{\sum_{m \in V_j} \mu(m)\}$.\;
        Compute gradient $\nabla E(\bm{h})$ and Hessian $H$ using Equation (\ref{eq:gradientEnergy}) and (\ref{eq:hessian}).\;
        $\bm{h} \leftarrow \bm{h} - \lambda H^{-1} \nabla E(\bm{h})$. // Update the minimizer $\bm{h}$ according to (\ref{eq:gradientSolver})\;
    }
	\KwRet{$V,\bm{h}$}.\;
}
\caption{Variational optimal transportation}
\label{alg:ot}
\end{algorithm}

Now, we introduce an energy function

\begin{equation}
\begin{split}
	E(\bm{h}) \eqdef & \int_{\Omega}^{}\theta_{\bm{h}}(x)\mu(x)dx -\sum_{j=1}^{k}\nu_i h_j\\
	\equiv & \int_{}^{\bm{h}}\sum_{j=1}^{k}w_j(\xi)d\xi -\sum_{j=1}^{k}\nu_i h_j.
	\label{eq:energy}
\end{split}
\end{equation}
\noindent $E$ is differentiable w.r.t. $\bm{h}$~\cite{gu2013variational}. Its gradient and Hessian are then given by 

\begin{equation}
	\nabla E(\bm{h}) = (w_1(\bm{h})-\nu_1, ..., w_k(\bm{h})-\nu_k)^T,
	\label{eq:gradientEnergy}
\end{equation}

\begin{equation}
	H = \frac{\partial^2E(\bm{h})}{\partial h_i\partial h_j} =
	\left\{
	\begin{array}{ll}
		 \sum_{l} \cfrac{\int_{f_{il}}\mu(x)dx}{\|y_l-y_i\|}, &\quad i = j,\ \forall l, s.t.\  f_{il} \neq \emptyset, \\[3ex]
		 -\cfrac{\int_{f_{ij}}\mu(x)dx}{\|y_j-y_i\|}, &\quad i \neq j,\ f_{ij} \neq \emptyset, \\[3ex]
		 0, &\quad i \neq j,\ f_{ij} = \emptyset,
	\end{array}
	\right.
	\label{eq:hessian}
\end{equation}
where $\|\cdot\|$ is the $L1$--norm and $\int_{f_{ij}}\mu(x)dx = \text{vol}(f_{ij})$ is the volume of the intersection $f_{ij}$ between two adjacent cells. Fig.~\ref{fig:hessian} (b) illustrates the geometric relation. The Hessian $H$ is positive semi-definite with only constant functions spanned by a vector $(1,...,1)^T$ in its null space. Thus, $E$ is strictly convex in $\bm{h}$. By Newton's method, we solve a linear system,

\begin{equation}
	H\delta\bm{h} = \nabla E(\bm{h}),
	\label{eq:gradientSolver}
\end{equation}
\noindent
and update $\bm{h}^{(t+1)} \leftarrow \bm{h}^{(t)} + \delta\bm{h}^{(t)}$. \ignore{The superscript $^{(t)}$ indicates iteration.} The energy $E$ (4) is motivated by Theorem~\ref{theory:alexandrov} which seeks a solution to $vol(x \in \Omega \mid \nabla \theta_{\bm{h}}(x)=y_j) = \nu_j$. Move the right-hand side to left and take the integral over $\bm{h}$ then it becomes $E$ (4). Thus, minimizing (4) when the gradient approaches $\bm{0}$ gives the solution. We show the complete algorithm for obtaining the OT-Map $\pi:X\rightarrow Y$ in Alg.~\ref{alg:ot}.

\begin{algorithm}[!t]
\SetKwFunction{FRecurs}{Iterative-Measure-Preserving-Mapping}
\Fn{\FRecurs{$X(x,\mu), Y(y,\nu)$}}{
\DontPrintSemicolon
	\Repeat{$y$ {\normalfont converges.}}{
	$V(\bm{h}) \leftarrow$ Variational-OT($x,\mu,y,\nu$). 	// 1. Update Voronoi partition\;
	$y_j \leftarrow \sum_{x\in V_j} \mu_i x_i  \big/ \sum_{x\in V_j} \mu_i$. 	// 2. Update $y$\;
    }
	\KwRet{$y,V$}.\;
}
\caption{Iterative measure-preserving mapping}
\label{alg:measurePreserving}
\end{algorithm}

\section{Variational Wasserstein Clustering}
\label{sec:vwc}
We now introduce in detail our method to solve clustering problems through variational optimal transportation. We name it \textit{variational Wasserstein clustering} (VWC). We focus on the semi-discrete clustering problem which is to find a set of discrete sparse centroids to best represent a continuous probability measure, or its discrete empirical representation. Suppose $M$ is a metric space and we embody in it an empirical measure $X(x,\mu)$. Our goal is to find such a sparse measure $Y(y,\nu)$ that minimizes Eq. (\ref{eq:kmeans}).

We begin with an assumption that the distributional data are embedded in the same Euclidean space $M = \mathbb{R}^n$, i.e. $X,Y \in \mathcal{P}(M)$. We observe that if $\nu$ is fixed then Eq. (\ref{eq:ot_monge}) and Eq. (\ref{eq:kmeans}) are mathematically equivalent. Thus, the computational approaches to these problems could also coincide. Because the space is convex, each cluster is eventually a Voronoi cell and the resulting partition $V=\{(V_j,y_j)\}$ is actually a power Voronoi diagram where we have $\|x-y_j\|^2-r_j^2 \leq \|x-y_i\|^2-r_i^2$, $x \in V_j$, $\forall j \neq i$ and $r$ is associated with the total mass of each cell. Such a diagram is also the solution to Monge's OT problem between $X$ and $Y$. From the previous section, we know that if $X$ and $Y$ are fixed the power diagram is entirely determined by the minimizer $\bm{h}$. Thus, assuming $\nu$ is fixed and $y$ is allowed to move freely in $M$, we reformulate Eq. (\ref{eq:kmeans}) to

\begin{equation}
\label{eq:impm}
	f(\bm{h},y) = \sum_{j=1}^{K}\sum_{x_i\in V_j(\bm{h})}\mu_i\|x_i-y_j\|^2,
\end{equation}
\noindent where every $V_j$ is a power Voronoi cell.

The solution to Eq. (\ref{eq:impm}) can be achieved by iteratively updating $\bm{h}$ and $y$. While we can use Alg.~\ref{alg:ot} to compute $\bm{h}$, updating $y$ can follow the rule:

\begin{equation}
	y_j^{(t+1)} \leftarrow \sum \mu_i x_i^{(t)}  \bigg/ \sum \mu_i,\ {x_i^{(t)}\in V_j}.
\end{equation}

\begin{algorithm}[!t]
\DontPrintSemicolon
\SetKwInOut{Input}{Input}\SetKwInOut{Output}{Output}
\Input{Empirical measures $X_M(x,\mu)$ and $Y_N(y,\nu)$}
\Output{Measure-preserving Map $\pi:X \rightarrow Y$ represented as ($y,V$).}
\Begin{
	$\nu \leftarrow$ Sampling-known-distribution. // Initialization.\;
	Harmonic-mapping: $M,N \rightarrow \mathbb{R}^n$ or $\mathbb{D}^n$. // Unify domains.\; 
	$y,V \leftarrow$ Iterative-Measure-Preserving-Mapping($x,\mu,y,\nu$).\;
}
\KwRet{$y,V$}.\;
\caption{Variational Wasserstein clustering}
\label{alg:wassersteinClustering}
\end{algorithm}

\noindent Since the first step preserves the measure and the second step updates the measure, we call such a mapping an \textit{iterative measure-preserving mapping}. Our algorithm repeatedly updates the partition of the space by variational-OT and computes the new centroids until convergence, as shown in Alg.~\ref{alg:measurePreserving}. Furthermore, because each step reduces the total cost (\ref{eq:impm}), we have the following propositions.

\begin{proposition}
\label{ps:minimize}
Alg.~\ref{alg:measurePreserving} monotonically minimizes the object function (\ref{eq:impm}).	
\end{proposition}

\begin{proof}
	It is sufficient for us to show that for any $t \geq 0$, we have
	\begin{equation}
		f(\bm{h}^{(t+1)},y^{(t+1)}) \leq f(\bm{h}^{(t)},y^{(t)}).
	\end{equation}
	The above inequality is indeed true since $f(\bm{h}^{(t+1)},y^{(t)}) \leq f(\bm{h}^{(t)},y^{(t)})$ according to the convexity of our OT formulation, and $f(\bm{h}^{(t+1)},y^{(t+1)}) \leq f(\bm{h}^{(t+1)},y^{(t)})$ for the updating process itself minimizes the mean squared error.
\end{proof}

\begin{corollary}
\label{ps:converge}
Alg.~\ref{alg:measurePreserving} converges in a finite number of iterations.
\end{corollary}

\begin{proof}
We borrow the proof for k-means. Given $N$ empirical samples and a fixed number $k$, there are $k^N$ ways of clustering. At each iteration, Alg.~\ref{alg:measurePreserving} produces a new clustering rule only based on the previous one. The new rule induces a lower cost if it is different than the previous one, or the same cost if it is the same as the previous one. Since the domain is a finite set, the iteration must eventually enter a cycle whose length cannot be greater than 1 because otherwise it violates the fact of the monotonically declining cost. Therefore, the cycle has the length of 1 in which case the Alg.~\ref{alg:measurePreserving} converges in a finite number of iterations. 
\end{proof}

\begin{corollary}
\label{ps:unique}
	Alg.~\ref{alg:measurePreserving} produces a unique (local) solution to Eq. (\ref{eq:impm}).
\end{corollary}
\begin{proof}
The initial condition, $y$ the centroid positions, is determined. Each step of Alg.~\ref{alg:measurePreserving} yields a unique outcome, whether updating $\bm{h}$ by variational OT or updating $y$ by weighted averaging. Thus, Alg.~\ref{alg:measurePreserving} produces a unique outcome.
\end{proof}

Now, we introduce the concept of variational Wasserstein clustering. For a subset $M \subset \mathbb{R}^n$, let $\mathcal{P}(M)$ be the space of all Borel probability measures. Suppose $X(x,\mu) \in \mathcal{P}(M)$ is an existing one and we are to aggregate it into $k$ clusters represented by another measure $Y(y,\nu) \in \mathcal{P}(M)$ and assignment $y_j = \pi(x)$, $j=1,...,k$. Thus, we have $\pi \in \mathcal{P}(M \times M)$. Given $\nu$ fixed, our goal is to find such a combination of $Y$ and $\pi$ that minimize the object function:

\begin{figure}[!t]
  \centering
  \begin{minipage}[b]{0.6\linewidth}
  \centering
    \includegraphics[width=\linewidth]{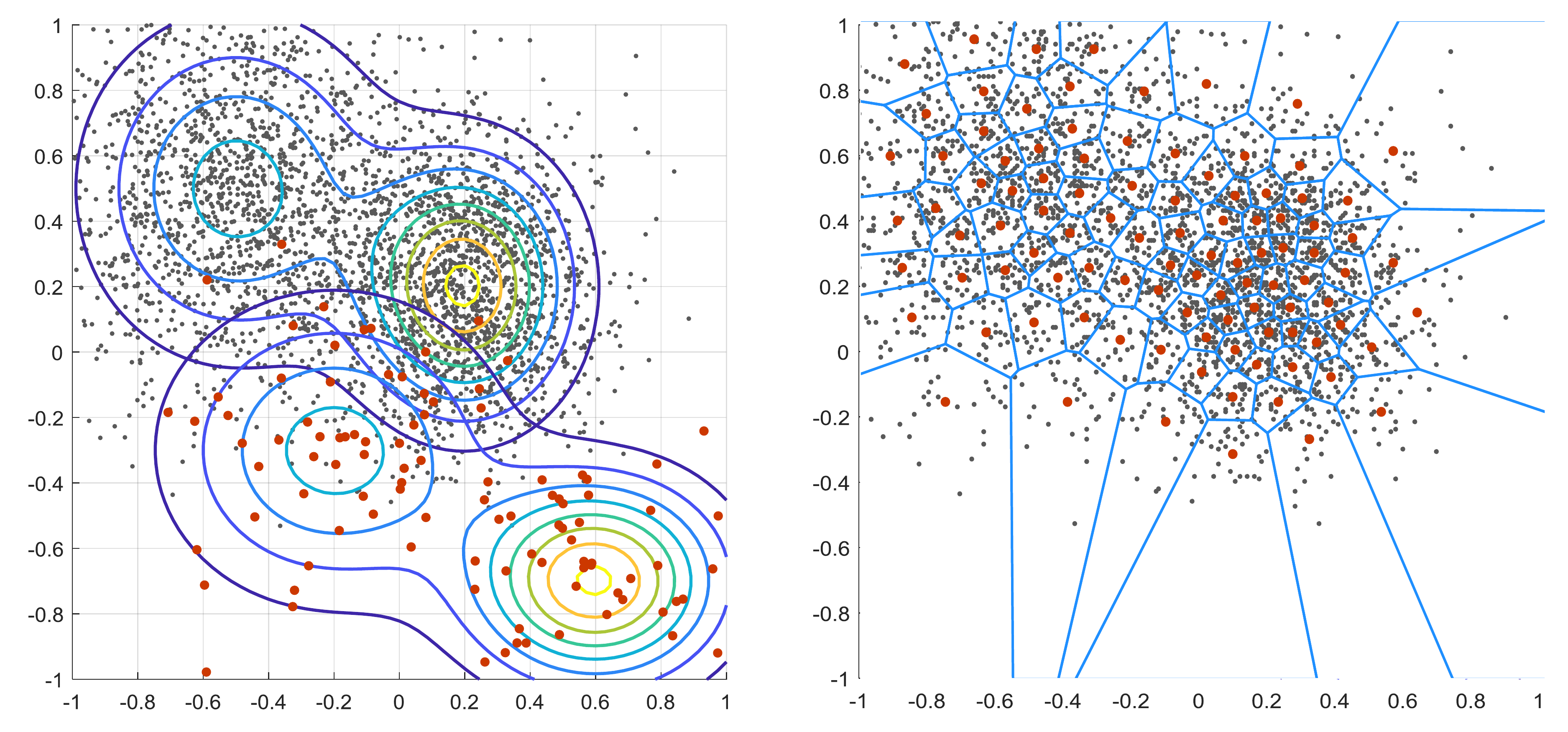}
    \caption{{\small Given the source domain (red dots) and target domain (grey dots), the distribution of the source samples are driven into the target domain and form a power Voronoi diagram.}}
    \label{fig:illustrate}
  \end{minipage}
  \hfill\vrule{}\hfill
  \begin{minipage}[b]{0.32\linewidth}
  \centering
    \includegraphics[width=\linewidth]{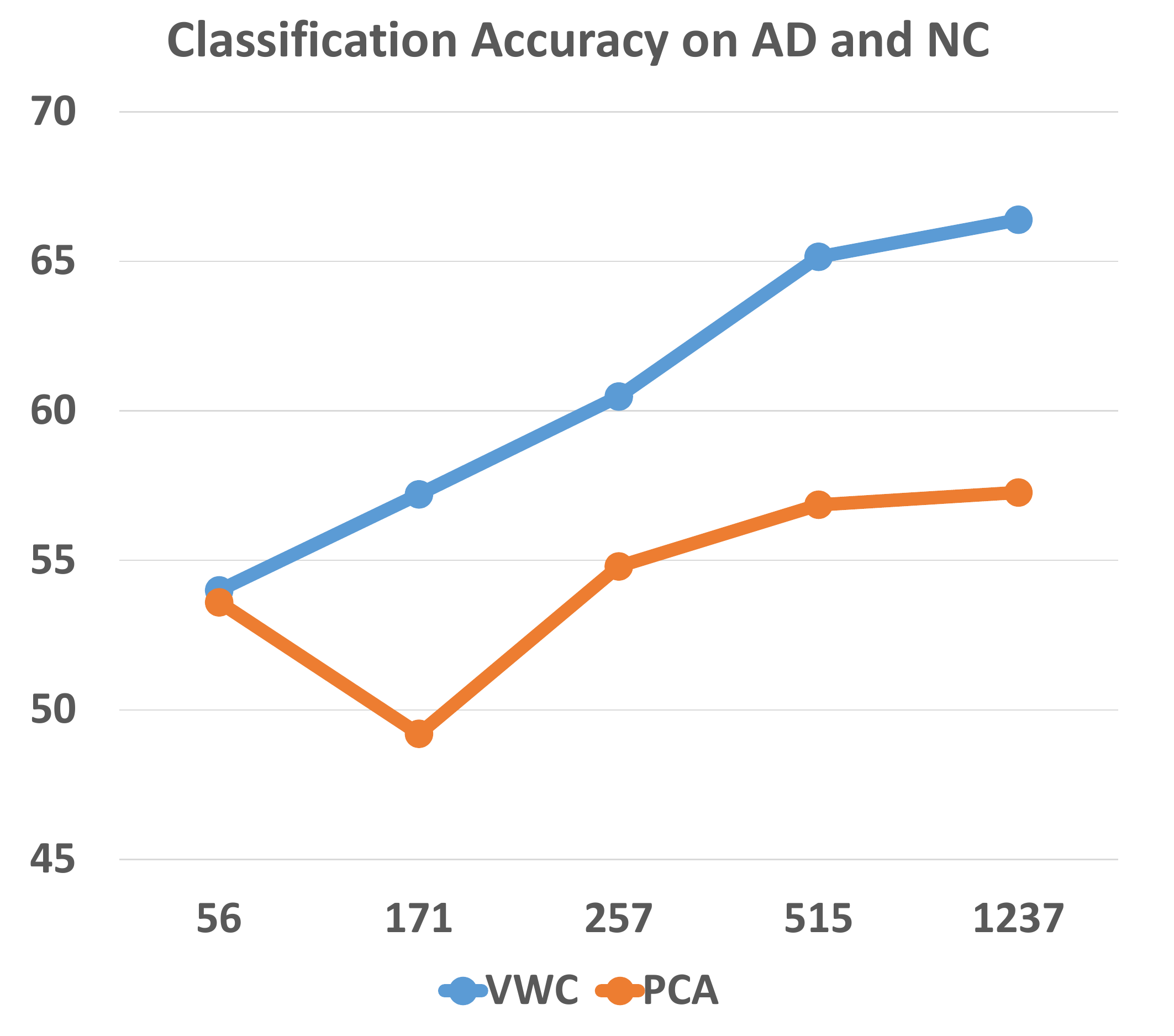}
    \caption{{\small Classification accuracies of VWC and PCA on AD and NC w.r.t. number of centroids. \ignore{\\\hspace{\textwidth}}}}
    \label{fig:brainPCA}
  \end{minipage}
\end{figure}

\begin{equation}
\label{eq:vwc}
	Y_{y,\nu} = \underset{\substack{Y \in P(M)\\
	\pi \in P(M \times M)}
	}{argmin} \sum_{j=1}^{k}\sum_{y_j=\pi(x_i)}\mu_i\|x_i-y_j\|^2,\ s.t.\ \nu_j = \sum_{y_j=\pi(x_i)}\mu_i.
\end{equation}

Eq. (\ref{eq:vwc}) is not convex w.r.t. $y$ as discussed in~\cite{cuturi2014fast}. We thus solve it by iteratively updating $\pi$ and $y$. When updating $\pi$, since $y$ is fixed, Eq. (\ref{eq:vwc}) becomes an optimal transportation problem. Therefore, solving Eq. (\ref{eq:vwc}) is equivalent to approaching the infimum of the $2$-Wasserstein distance between $X$ and $Y$: 
\begin{equation}
\label{eq:vwc2}
	\underset{\substack{Y \in P(M)\\
	\pi \in P(M \times M)}
	}{\inf} \sum_{j=1}^{k}\sum_{y_j=\pi(x_i)}\mu_i\|x_i-y_j\|^2\ = 	\underset{Y \in P(M)}{\inf} W_{2}^{2}(X,Y).
\end{equation}

\noindent Assuming the domain is convex, we can apply iterative measure-preserving mapping (Alg.~\ref{alg:measurePreserving}) to obtain $y$ and $h$ which induces $\pi$. In case that $X$ and $Y$ are not in the same domain i.e. $Y(y,\nu) \in P(N)$, $N \subset \mathbb{R}^n$, $N \neq M$, or the domain is not necessarily convex, we leverage \textit{harmonic mapping}~\cite{gu2008computational,wang2004volumetric} to map them to a convex canonical space. We wrap up our complete algorithm in Alg.~\ref{alg:wassersteinClustering}. Fig.~\ref{fig:illustrate} illustrates a clustering result. Given a source Gaussian mixture (red dots) and a target Gaussian mixture (grey dots), we cluster the target domain with the source samples. Every sample has the same mass in each domain for simplicity. Thus, we obtain an unweighted Voronoi diagram. In the next section, we will show examples that involve different mass. We implement our algorithm in C/C++ and adopt Voro++~\cite{rycroft2009voro++} to compute Voronoi diagrams. The code is available at \href{https://github.com/icemiliang/vwc}{https://github.com/icemiliang/vot}.


\section{Applications}
\label{sec:apply}
While the k-means clustering problem is ubiquitous in numerous tasks in computer vision and machine learning, we present the use of our method in approaching domain adaptation, remeshing, and representation learning.

\subsection{Domain Adaptation on Synthetic Data}
\label{sec:apply_domain}
Domain adaptation plays a fundamental role in knowledge transfer and has benefited many different fields such as scene understanding and image style transfer. Several works have coped with domain adaptation by transforming distributions in order to close their gap with respect to a measure. In recent years, Courty \MakeLowercase{\textit{et al.}}~\cite{courty2014domain} took the first steps in applying optimal transportation to domain adaptation. Here we revisit this idea and provide our own solution to \textit{unsupervised many-to-one domain adaptation} based on variational Wasserstein clustering.

\begin{figure}[!t]
  \centering
    \includegraphics[width=\linewidth]{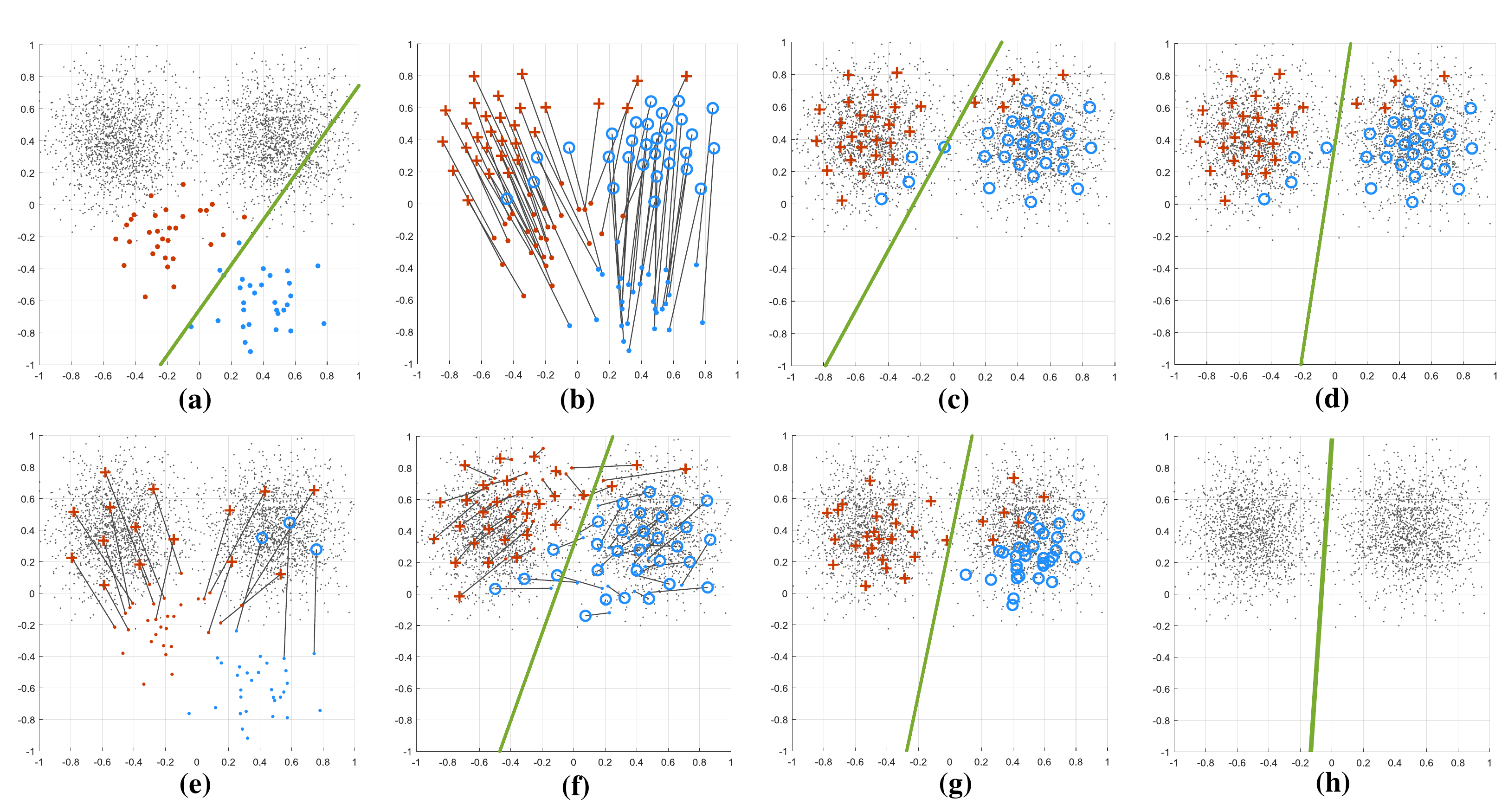}
    \caption{{\small SVM RBF boundaries for domain adaptation. (a) Target domain in gray dots and source domain of two classes in red and blue dots; (b) Mapping of centroids by using VWC; (c,d) Boundaries from VWC with linear and RBF kernels; (e) k-means++~\cite{arthur2007k} fails to produce a model; (f) After recentering source and target domains, k-means++ yields acceptable boundary; (g) D2~\cite{ye2017fast}; (h) JDOT~\cite{courty2017joint}, final centroids not available.}}
    \label{fig:da}
\end{figure}

Consider a two-class classification problem in the 2D Euclidean space. The source domain consists of two independent Gaussian distributions sampled by red and blue dots as shown in Fig.~\ref{fig:da} (a). Each class has 30 samples. The target domain has two other independent Gaussian distributions with different means and variances, each having 1500 samples. They are represented by denser gray dots to emulate the source domain after an unknown transformation. 

We adopt support vector machine (SVM) with linear and radial basis function (RBF) kernels for classification. The kernel scale for RBF is 5. One can notice that directly applying the RBF classifier learned from the source domain to the target domain provides a poor classification result (59.80\%). While Fig.~\ref{fig:da} (b) shows the final positions of the samples from the source domain by VWC, (c) and (d) show the decision boundaries from SVMs with a linear kernel and an RBF kernel, respectively. In (e) and (f) we show the results from the classic k-means++ method~\cite{arthur2007k}. In (e) k-means++ fails to cluster the unlabeled samples into the original source domain and produces an extremely biased model that has 50\% of accuracy. Only after we recenter the source and the target domains yields k-means++ better results as shown in (f). 

For more comparison, we test two other methods -- D2~\cite{ye2017fast} and JDOT~\cite{courty2017joint}. The final source positions from D2 are shown in (g). Because D2 solves the general barycenter problem and also updates the weights of the source samples, it converges as soon as it can find them some positions when the weights can also satisfy the minimum clustering loss. Thus, in (g), most of the source samples dive into the right, closer density, leaving those moving to the left with larger weights. We show the decision boundary obtained from JDOT~\cite{courty2017joint} in (h). JDOT does not update the centroids, so we only show its decision boundary. In this experiment, both our method for Monge's OT and the methods~\cite{courty2017joint,ye2017fast} for Kantorovich's OT can effectively transfer knowledge between different domains, while the traditional method~\cite{arthur2007k} can only work after a prior knowledge between the two domains, e.g. a linear offset. Detailed performance is reported in Tab.~\ref{tab:domain}.  

\bgroup
\begin{table}[!b]
\centering
\caption{Classification Accuracy for Domain Adaptation on Synthetic Data} \label{tab:domain}
\begin{tabular}{@{\ }c@{\ }|@{\ }c@{\ }|@{\ }c@{\ }c@{\ }|c@{\ }c@{\ }|c@{\ }c@{\ }|c@{\ }c@{\ }}
\hline
{}& {k-means++~\cite{arthur2007k}$\bm{^*}$} & \multicolumn{2}{c|}{k-means++$\bm{^r}$} & \multicolumn{2}{|c|}{D2~\cite{ye2017fast}} & \multicolumn{2}{|c|}{JDOT~\cite{courty2017joint}} & \multicolumn{2}{|c}{VWC}\\
\hline
Kernel      & Linear/RBF & Linear & RBF    & Linear & RBF   & Linear & RBF   & Linear & RBF \\
\hline
Acc.    & 50.00      & 97.88  & 99.12  & 95.85  & 99.25 & 99.03  & 99.23 & 98.56  & 99.31 \\
\hline
Sen. & 100.00     & 98.13 & 98.93   & 99.80  & 99.07 & 98.13  & 99.60 & 98.00  & 99.07\\
\hline
Spe. & 0.00       & 97.53 & 99.27   & 91.73  & 99.40 & 99.93  & 98.87 & 99.07  & 99.53\\
\hline
\end{tabular}
\begin{flushleft} 
$\bm{^*}$: extremely biased model labeling all samples with same class; $\bm{^r}$: after recenterd.
\end{flushleft}
\end{table}
\egroup

\subsection{Deforming Triangle Meshes}
\label{sec:apply_remesh}

Triangle meshes is a dominant approximation of surfaces. Refining triangle meshes to best represent surfaces have been studied for decades, including~\cite{shewchuk2002delaunay,fabri2009cgal,goes2014weighted}. Given limited storage, we prefer to use denser and smaller triangles to represent the areas with relatively complicated geometry and sparser and larger triangles for flat regions. We follow this direction and propose to use our method to solve this problem. The idea is to drive the vertices toward high-curvature regions. 

\begin{figure}[!t]
  \centering
    \includegraphics[width=\linewidth]{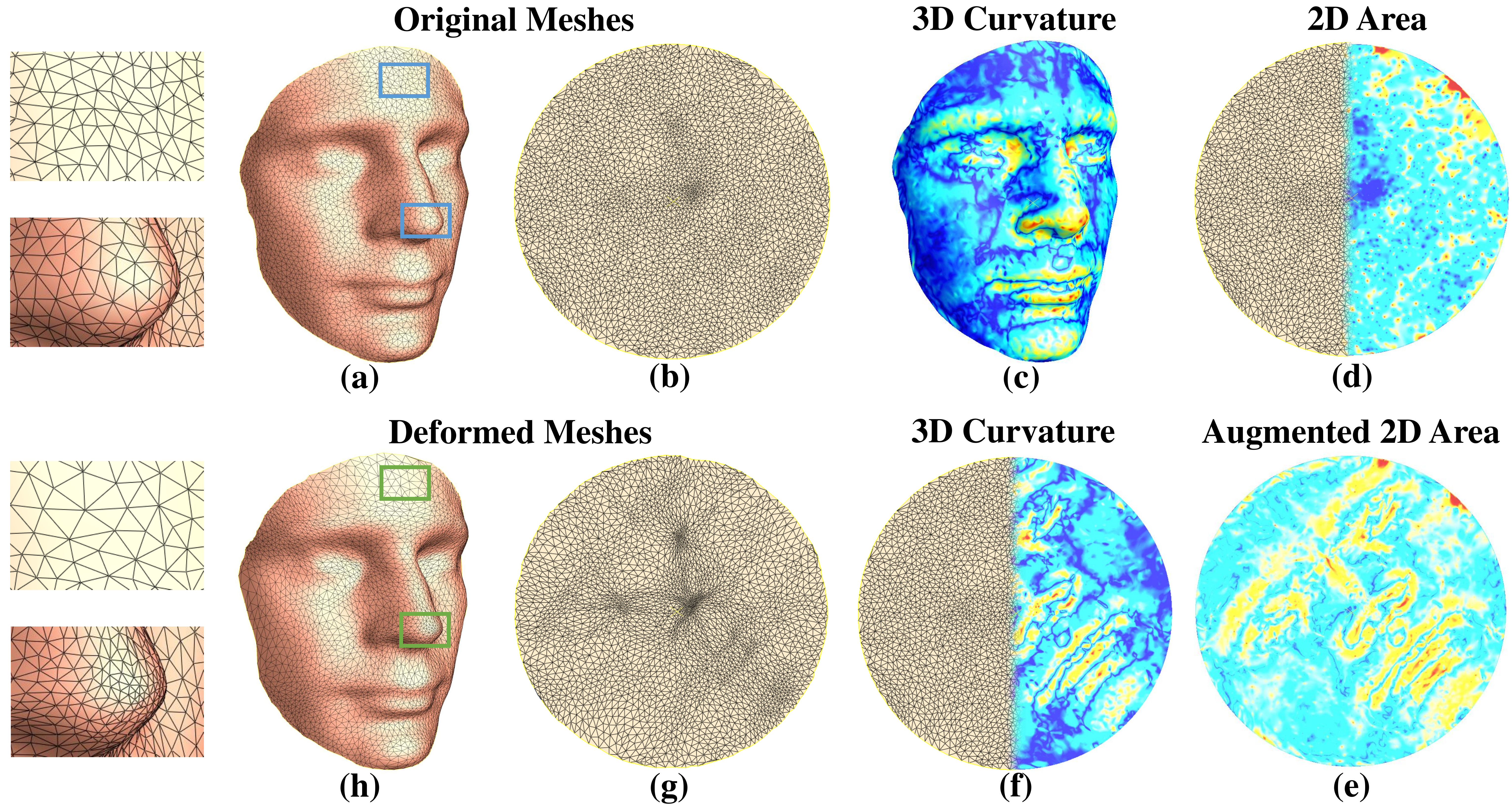}
    \caption{{\small Redistribute triangulation based on curvature. Original mesh (a) is mapped to a unit disk (b). Mean curvature on the 3D mesh (c) is copied to the disk (f). Design an ``augmented'' measure $\mu$ (e) on the disk by incorporating curvature $C$ into 2D vertex area $A$ (d), e.g. $\mu = 0.4 A + 0.6 C$. A vertex $y$ with a large curvature $C$, in order to maintain its original measure $A$, will shrink its own cluster. As a result vertices collapse in high-curvature regions (g). Mesh will be pulled back to 3D (h) by inverse mapping.}}
    \label{fig:rm}
\end{figure}

We consider a surface $\mathbb{S}^2$ approximated by a triangle mesh $T_{\mathbb{S}^2}(v)$. To drive the vertices to high-curvature positions, our general idea is to reduce the areas of the triangles in there and increase them in those locations of low curvature, producing a new triangulation $T_{\mathbb{S}^2}'(v)$ on the surface. To avoid computing the geodesic on the surface, we first map the surface to a \textit{unit disk} $\phi:\mathbb{S}^2\rightarrow \mathbb{D}^2 \subset \mathbb{R}^2$ and equip it with the Euclidean metric. We drop the superscripts $2$ for simplicity. To clarify notations, we use $T_{\mathbb{S}}(v)$ to represent the original triangulation on surface $\mathbb{S}$; $T_{\mathbb{D}}(v)$ to represent its counterpart on $\mathbb{D}$ after harmonic mapping; $T_{\mathbb{D}}'(v)$ for the target triangulation on $\mathbb{D}$ and $T_{\mathbb{S}}'(v)$ on $\mathbb{S}$. Fig.~\ref{fig:rm} (a) and (b) illustrate the triangulation before and after the harmonic mapping. Our goal is to rearrange the triangulation on $\mathbb{D}$ and then the following composition gives the desired triangulation on the surface:

\begin{equation}
	T_{\mathbb{S}}(v) \xrightarrow{\ \phi\ } T_{\mathbb{D}}(v) \xrightarrow{\ \pi\ } T_{\mathbb{D}}'(v) \xrightarrow{\ \phi^{-1}\ } T_{\mathbb{S}}'(v).\nonumber
\end{equation}

\noindent $\pi$ is where we apply our method. 

Suppose we have an original triangulation $T_{sub,\mathbb{S}}(v)$ and an initial downsampled version $T_{\mathbb{S}}(v)$ and we map them to $T_{sub,\mathbb{D}}(v)$ and $T_{\mathbb{D}}(v)$, respectively. The vertex area $A_{\mathbb{D}}: v \rightarrow a$ on $\mathbb{D}$ is the source (Dirac) measure. We compute the (square root of absolute) mean curvature $C_{sub,\mathbb{S}}: v_{sub} \rightarrow c_{sub}$ on $\mathbb{S}$ and the area $A_{sub,\mathbb{D}}: v_{sub} \rightarrow a_{sub}$ on $\mathbb{D}$. After normalizing $a$ and $c$, a weighted summation gives us the target measure, $\mu_{sub,\mathbb{D}} = (1-\lambda)\ a_{sub,\mathbb{D}} + \lambda\ c_{sub,\mathbb{D}}$. We start from the source measure $(v,a)$ and cluster the target measure $(v_{sub},\mu_{sub})$. The intuition is the following. If $\lambda = 0$, $\mu_{i,sub} = a_{i,sub}$ everywhere, then a simple unweighted Voronoi diagram which is the dual of $T_{\mathbb{D}}(v)$ would satisfy Eq. (\ref{eq:vwc2}). As $\lambda$ increases, the clusters $V_j(v_j,a_j)$ in the high-curvature ($c_{sub,\mathbb{D}}$) locations will require smaller areas ($a_{sub,\mathbb{D}}$) to satisfy $a_j = \sum_{v_{i,sub} \in V_j} \mu_{i,sub}$.

We apply our method on a human face for validation and show the result in Fig.~\ref{fig:rm}. On the left half, we show the comparison before and after the remeshing. The tip of the nose has more triangles due to the high curvature while the forehead becomes sparser because it is relatively flatter. The right half of Fig.~\ref{fig:rm} shows different measures that we compute for moving the vertices. (c) shows the mean curvature on the 3D surface. We map the triangulation with the curvature onto the planar disk space (f). (d) illustrates the vertex area of the planar triangulation and (e) is the weighted combination of 3D curvature and 2D area. Finally, we regard area (d) as the source domain and the ``augmented'' area (e) as the target domain and apply our method to obtain the new arrangement (g) of the vertices on the disk space. After that, we pull it back to the 3D surface (h). As a result, vertices are attracted into high-curvature regions. Note the boundaries of the deformed meshes (g,h) have changed after the clustering. We could restrict the boundary vertices to move \textit{on} the unit circle if necessary. Rebuilding a Delaunay triangulation from the new vertices is also an optional step after.

\subsection{Learning Representations of Brain Images}
\label{sec:apply_represent}

Millions of voxels contained in a 3D brain image bring efficiency issues for computer-aided diagnoses. A good learning technique can extract a better representation of the original data in the sense that it reduces the dimensionality and/or enhances important information for further processes. In this section, we address learning representations of brain images from a perspective of Wasserstein clustering and verify our algorithm on magnetic resonance imaging (MRI).

\begin{figure}[!t]
  \centering
    \includegraphics[width=\linewidth]{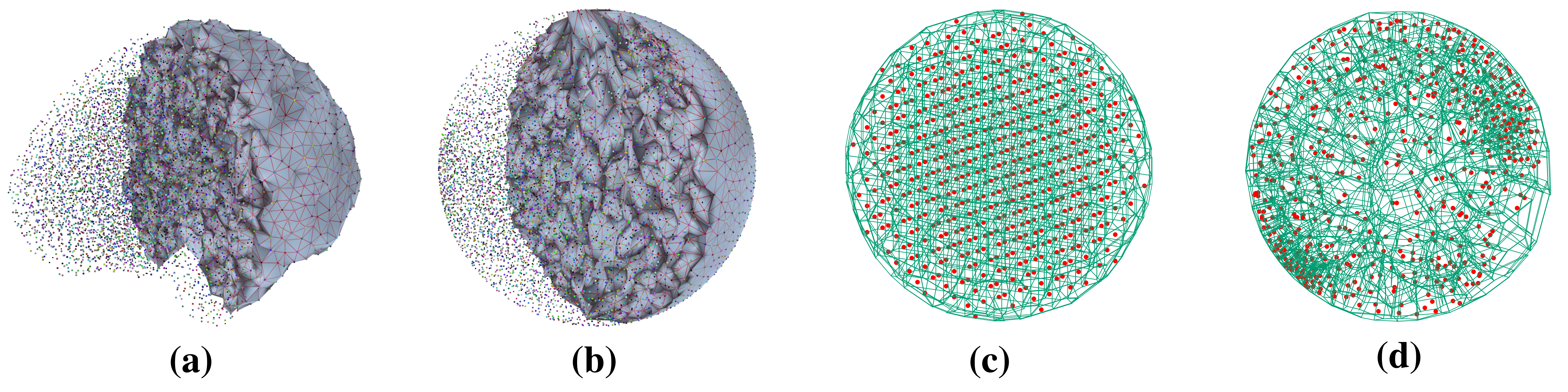}
    \caption{{\small Brain images are projected to tetrahedral meshes (a) that are generated from brain surfaces. Meshes are deformed into unit balls (b) via harmonic mapping. A sparse uniform distribution inside the ball (c) is initialized and shown with the initial Voronoi diagram. We start from (c) as the initial centroids and cluster (b) as the empirical measure by using our proposed method. (d) shows the resulting centroids and diagram.}}
\label{fig:rl}
\end{figure}

In the high level, given a brain image $X(x,\mu)$ where $x$ represents the voxels and $\mu$ their intensities, we aim to cluster it with a known sparse distribution $Y(y,\nu)$. \ignore{Brain images, particularly MRI, can be treated as distributions. MRI measures the time taken for hydrogen nuclei to return to their original state after excited by external radiofrequency (RF) pulses~\cite{mcrobbie2017mri} and thus can be considered as a map of the nuclear energy of the entire brain volume.} We consider that each brain image \ignore{representing the brain volume}is a submanifold in the 3D Euclidean space, $M \subset \mathbb{R}^3$. To prepare the data, for each image, we first remove the skull and extract the surface of the brain volume by using Freesurfer~\cite{fischl2012freesurfer}, and use Tetgen~\cite{si2006quality} to create a tetrahedral mesh from the surface. Then, we project the image onto the tetrahedral mesh by finding the nearest neighbors and perform harmonic mapping to deform it into a \textit{unit ball} as shown in Fig.~\ref{fig:rl} (a) and (b).
 
Now, following Alg.~\ref{alg:wassersteinClustering}, we set a discrete uniform distribution sparsely supported in the unit ball, $Y(y,\nu) \sim U_{\mathbb{D}^3}(-1,1)$ as shown in Fig.~\ref{fig:rl} (c). Starting from this, we learn such a new $y$ that the representation mapping $\pi:x \rightarrow y$ has the minimum cost (\ref{eq:vwc2}). Thus, we can think of this process as a non-parametric mapping from the input to a latent space $\mathcal{P}(y)$ of dimension $k \times n \ll |x|$ where $k$ is the number of clusters and $n$ specifies the dimension of the original embedding space, e.g. $3$ for brain images. Fig.~\ref{fig:rl} (d) shows the resulting centroids and the corresponding power diagram. We compare our method with principle component analysis (PCA) to show its capacity in dimensionality reduction. We apply both methods on 100 MRI images with 50 of them labeled Alzheimer's disease (AD) and 50 labeled normal control (NC). After obtaining the low-dimensional features, we directly apply a linear SVM classifier on them for 5-fold cross-validation. The plots in Fig.~\ref{fig:brainPCA} show the superiority of our method. It is well known that people with AD suffer brain atrophy resulting in a group-wise shift in the images~\cite{fox1997brain}. The result shows the potential of VWC in embedding the brain image in low-dimensional spaces. We could further incorporate prior knowledge such as regions-of-interest into VWC by hand-engineering initial centroids.


\section{Discussion}
\label{sec:discuss}
Optimal transportation has gained increasing popularity in recent years thanks to its robustness and scalability in many areas. In this paper, we have discussed its connection to k-means clustering. Built upon variational optimal transportation, we have proposed a clustering technique by solving iterative measure-preserving mapping and demonstrated its applications to domain adaptation, remeshing, and learning representations.

One limitation of our method at this point is computing a high-dimensional Voronoi diagram. It requires complicated geometry processing which causes efficiency and memory issues. A workaround of this problem is to use gradient descent for variational optimal transportation because the only thing we need from the diagram is the intersections of adjacent convex hulls for computing the Hessian. The assignment of each empirical observation obtained from the diagram can be alternatively determined by nearest search algorithms. This is beyond the scope of this paper but it could lead to more real-world applications.

The use of our method for remeshing could be extended to the general feature redistribution problem on a compact $2$--manifold. Future work could also include adding regularization to the centroid updating process to expand its applicability to specific tasks in computer vision and machine learning. The extension of our formulation of Wasserstein means to barycenters is worth further study.\\

\noindent \textbf{Acknowledgemants} The research is partially supported by National Institutes of Health (R21AG043760, RF1AG051710, and R01EB025032), and National Science Foundation (DMS-1413417 and IIS-1421165).

\clearpage

\bibliographystyle{splncs}
\bibliography{egbib}
\end{document}